\documentclass[letterpaper]{article}
\usepackage[T1]{fontenc}
\usepackage{array}
\usepackage{float}
\usepackage{wrapfig}
\usepackage{booktabs}
\usepackage{url}
\usepackage{multirow}
\usepackage{algorithm2e}
\usepackage{amsmath}
\usepackage{amsthm}
\usepackage{graphicx}
\usepackage{microtype}
\usepackage[unicode=true,
 bookmarks=false,
 breaklinks=false,pdfborder={0 0 1},backref=section,colorlinks=false]
 {hyperref}

\makeatletter

\pdfpageheight\paperheight
\pdfpagewidth\paperwidth

\providecommand{\tabularnewline}{\\}

\theoremstyle{plain}
\newtheorem{thm}{\protect\theoremname}
\ifx\proof\undefined
\newenvironment{proof}[1][\protect\proofname]{\par
	\normalfont\topsep6\p@\@plus6\p@\relax
	\trivlist
	\itemindent\parindent
	\item[\hskip\labelsep\scshape #1]\ignorespaces
}{%
	\endtrivlist\@endpefalse
}
\providecommand{\proofname}{Proof}
\fi

\usepackage{placeins}

\usepackage[nonatbib,final]{nips_2018}

\usepackage{algorithmic}
\usepackage{url}
\usepackage{amsfonts}
\usepackage{nicefrac}

\title{Regularization Learning Networks: Deep Learning for Tabular Datasets}

\author{
 Ira Shavitt\\
 Weizmann Institute of Science\\
 \texttt{irashavitt@gmail.com} \\
  \And
  Eran Segal \\
  Weizmann Institute of Science \\
  \texttt{eran.segal@weizmann.ac.il} \\
}

\@ifundefined{showcaptionsetup}{}{%
 \PassOptionsToPackage{caption=false}{subfig}}
\usepackage{subfig}
\makeatother

\providecommand{\theoremname}{Theorem}

\begin{document}

\maketitle
\begin{abstract}
Despite their impressive performance, \textit{Deep Neural Networks
}(DNNs) typically underperform \textit{Gradient Boosting Trees }(GBTs)
on many tabular-dataset learning tasks. We propose that applying a
different regularization coefficient to each weight might boost the
performance of DNNs by allowing them to make more use of the more
relevant inputs. However, this will lead to an intractable number
of hyperparameters. Here, we introduce \textit{Regularization Learning
Networks} (RLNs), which overcome this challenge by introducing an
efficient hyperparameter tuning scheme which minimizes a new \textit{Counterfactual
Loss}. Our results show that RLNs significantly improve DNNs on tabular
datasets, and achieve comparable results to GBTs, with the best performance
achieved with an ensemble that combines GBTs and RLNs. RLNs produce
extremely sparse networks, eliminating up to $99.8\%$ of the network
edges and $82\%$ of the input features, thus providing more interpretable
models and reveal the importance that the network assigns to different
inputs. RLNs could efficiently learn a single network in datasets
that comprise both tabular and unstructured data, such as in the setting
of medical imaging accompanied by electronic health records. An open
source implementation of RLN can be found at \url{https://github.com/irashavitt/regularization_learning_networks}.
\end{abstract}

\section{Introduction}

Despite their impressive achievements on various prediction tasks
on datasets with distributed representation \cite{hintondistributed,Bengio,Bengio2007}
such as images \cite{Krizhevsky}, speech \cite{SOTA_speech}, and
text \cite{SOTA_machine_translation}, there are many tasks in which
\textit{Deep Neural Networks} (DNNs) underperform compared to other
models such as \textit{Gradient Boosting Trees} (GBTs). This is evident
in various \textit{Kaggle} \cite{Beam2015,Belkhayat2018}, or \textit{KDD}
\textit{Cup} \cite{Cai,Huang,Sandulescu} competitions, which are
typically won by GBT-based approaches and specifically by its \textit{XGBoost}
\cite{Chen} implementation, either when run alone or within a combination
of several different types of models.

The datasets in which neural networks are inferior to GBTs typically
have different statistical properties. Consider the task of image
recognition as compared to the task of predicting the life expectancy
of patients based on electronic health records. One key difference
is that in image classification, many pixels need to change in order
for the image to depict a different object \cite{Papernot}.\footnote{This is not contradictory to the existence of adversarial examples
\cite{Goodfellowb}, which are able to fool DNNs by changing a small
number of input features, but do not actually depict a different object,
and generally are not able to fool humans.} In contrast, the relative contribution of the input features in the
electronic health records example can vary greatly: Changing a single
input such as the age of the patient can profoundly impact the life
expectancy of the patient, while changes in other input features,
such as the time that passed since the last test was taken, may have
smaller effects.

We hypothesized that this potentially large variability in the relative
importance of different input features may partly explain the lower
performance of DNNs on such tabular datasets \cite{Goodfellow-et-al-2016}.
One way to overcome this limitation could be to assign a different
regularization coefficient to every weight, which might allow the
network to accommodate the non-distributed representation and the
variability in relative importance found in tabular datasets. 

This will require tuning a large number of hyperparameters. The default
approach to hyperparameter tuning is using derivative-free optimization
of the validation loss, i.e., a loss of a subset of the training set
which is not used to fit the model. This approach becomes computationally
intractable very quickly.

Here, we present a new hyperparameter tuning technique, in which we
optimize the regularization coefficients using a newly introduced
loss function, which we term the \textit{Counterfactual Loss}, or$\mathcal{L}_{CF}$.
We term the networks that apply this technique \textit{Regularization
Learning Networks }(RLNs). In RLNs, the regularization coefficients
are optimized together with learning the network weight parameters.
We show that RLNs significantly and substantially outperform DNNs
with other regularization schemes, and achieve comparable results
to GBTs. When used in an ensemble with GBTs, RLNs achieves state of
the art results on several prediction tasks on a tabular dataset with
varying relative importance for different features.

\section{Related work}

Applying different regularization coefficients to different parts
of the network is a common practice. The idea of applying different
regularization coefficients to every weight was introduced \cite{Maclaurin},
but it was only applied to images with a toy model to demonstrate
the ability to optimize many hyperparameters.

Our work is also related to the rich literature of works on hyperparameter
optimization \cite{hyperparameter_optimization_guide}. These works
mainly focus on derivative-free optimization \cite{bayesian_hyperparameter_optimization,greedy_and_random_hyperparameter_tuning,ML_to_perform_hyperparameter_tuning}.
Derivative-based hyperparameter optimization is introduced in \cite{hyperparameter_tuning_using_gradient}
for linear models and in \cite{Maclaurin} for neural networks. In
these works, the hyperparameters are optimized using the gradients
of the validation loss. Practically, this means that every optimization
step of the hyperparameters requires training the whole network and
back propagating the loss to the hyperparameters. \cite{Lorraine2018}
showed a more efficient derivative based way for hyperparameter optimization,
which still required a substantial amount of additional parameters.
\cite{basically_the_same} introduce an optimization technique similar
to the one introduced in this paper, however, the optimization technique
in \cite{basically_the_same} requires a validation set, and only
optimizes a single regularization coefficient for each layer, and
at most 10-20 hyperparameters in any network. In comparison, training
RLNs doesn't require a validation set, assigns a different regularization
coefficient for every weight, which results in up to millions of hyperparameters,
optimized efficiently. Additionally, RLNs optimize the coefficients
in the log space and adds a projection after every update to counter
the vanishing of the coefficients. Most importantly, the efficient
optimization of the hyperparameters was applied to images and not
to dataset with non-distributed representation like tabular datasets.

DNNs have been successfully applied to tabular datasets like electronic
health records, in \cite{Rajkomar2018,Miotto2016}. The use of RLN
is complementary to these works, and might improve their results and
allow the use of deeper networks on smaller datasets.

To the best of our knowledge, our work is the first to illustrate
the statistical difference in distributed and non-distributed representations,
to hypothesize that addition of hyperparameters could enable neural
networks to achieve good results on datasets with non-distributed
representations such as tabular datasets, and to efficiently train
such networks on a real-world problems to significantly and substantially
outperform networks with other regularization schemes.

\section{Regularization Learning\label{sec:Regularization-Learning}}

Generally, when using regularization, we minimize $\tilde{\mathcal{L}}\left(Z,W,\lambda\right)=\mathcal{L}\left(Z,W\right)+\exp\left(\lambda\right)\cdot\sum_{i=1}^{n}\left\Vert w_{i}\right\Vert $,
where $Z=\left\{ \left(x_{m},y_{m}\right)\right\} _{m=1}^{M}$ are
the training samples, $\mathcal{L}$ is the loss function, $W=\left\{ w_{i}\right\} _{i=1}^{n}$
are the weights of the model, $\left\Vert \cdot\right\Vert $ is some
norm, and $\lambda$ is the regularization coefficient,\footnote{The notation for the regularization term is typically $\lambda\cdot\sum_{i=1}^{n}\left\Vert w_{i}\right\Vert $.
We use the notation $\exp\left(\lambda\right)\cdot\sum_{i=1}^{n}\left\Vert w_{i}\right\Vert $
to force the coefficients to be positive, to accelerate their optimization
and to simplify the calculations shown.} a hyperparameter of the network. Hyperparameters of the network,
like $\lambda$, are usually obtained using cross-validation, which
is the application of derivative-free optimization on $\mathcal{L}_{CV}\left(Z_{t},Z_{v},\lambda\right)$
with respect to $\lambda$ where $\mathcal{L}_{CV}\left(Z_{t},Z_{v},\lambda\right)=\mathcal{L}\left(Z_{v},\arg\min_{W}\tilde{\mathcal{L}}\left(Z_{t},W,\lambda\right)\right)$
and $\left(Z_{t},Z_{v}\right)$ is some partition of $Z$ into train
and validation sets, respectively.

If a different regularization coefficient is assigned to each weight
in the network, our learning loss becomes $\mathcal{L}^{\dagger}\left(Z,W,\Lambda\right)=\mathcal{L}\left(Z,W\right)+\sum_{i=1}^{n}\exp\left(\lambda_{i}\right)\cdot\left\Vert w_{i}\right\Vert $,
where $\Lambda=\left\{ \lambda_{i}\right\} _{i=1}^{n}$ are the regularization
coefficients. Using $\mathcal{L}^{\dagger}$ will require $n$ hyperparameters,
one for every network parameter, which makes tuning with cross-validation
intractable, even for very small networks. We would like to keep using
$\mathcal{L}^{\dagger}$ to update the weights, but to find a more
efficient way to tune $\Lambda$. One way to do so is through SGD,
but it is unclear which loss to minimize: $\mathcal{L}$ doesn't have
a derivative with respect to $\Lambda$, while $\mathcal{L}^{\dagger}$
has trivial optimal values, $\arg\min_{\Lambda}\mathcal{L}^{\dagger}\left(Z,W,\Lambda\right)=\left\{ -\infty\right\} _{i=1}^{n}$.
$\mathcal{L}_{CV}$ has a non-trivial dependency on $\Lambda$, but
it is very hard to evaluate $\frac{\partial\mathcal{L}_{CV}}{\partial\Lambda}$.

We introduce a new loss function, called the \textit{Counterfactual
Loss} $\mathcal{L}_{CF}$, which has a non-trivial dependency on $\Lambda$
and can be evaluated efficiently. For every time-step $t$ during
the training, let $W_{t}$ and $\Lambda_{t}$ be the weights and regularization
coefficients of the network, respectively, and let $w_{t,i}\in W_{t}$
and $\lambda_{t,i}\in\Lambda_{t}$ be the weight and the regularization
coefficient of the same edge $i$ in the network. When optimizing
using SGD, the value of this weight in the next time-step will be
$w_{t+1,i}=w_{t,i}-\eta\cdot\frac{\partial\mathcal{L}^{\dagger}\left(Z_{t},W_{t},\Lambda_{t}\right)}{\partial w_{t,i}}$,
where $\eta$ is the learning rate, and $Z_{t}$ is the training batch
at time $t$.\footnote{We assume vanilla SGD is used in this analysis for brevity, but the
analysis holds for any derivative-based optimization method.} We can split the gradient into two parts:
\begin{align}
w_{t+1,i} & =w_{t,i}-\eta\cdot\left(g_{t,i}+r_{t,i}\right)\label{eq:w_t+1}\\
g_{t,i} & =\frac{\partial\mathcal{L}\left(Z_{t},W_{t}\right)}{\partial w_{t,i}}\\
r_{t,i} & =\frac{\partial}{\partial w_{t,i}}\left(\sum_{j=1}^{n}\exp\left(\lambda_{t,j}\right)\cdot\left\Vert w_{t,j}\right\Vert \right)=\exp\left(\lambda_{t,i}\right)\cdot\frac{\partial\left\Vert w_{t,i}\right\Vert }{\partial w_{t,i}}\label{eq:r}
\end{align}
We call $g_{t,i}$ the gradient of the empirical loss $\mathcal{L}$
and $r_{t,i}$ the gradient of the regularization term. All but one
of the addends of $r_{t,i}$ vanished since $\frac{\partial}{\partial w_{t,i}}\left(\exp\left(\lambda_{t,j}\right)\cdot\left\Vert w_{t,j}\right\Vert \right)=0$
for every $j\ne i$. Denote by $W_{t+1}=\left\{ w_{t+1,i}\right\} _{i=1}^{n}$
the weights in the next time-step, which depend on $Z_{t}$, $W_{t}$,
$\Lambda_{t}$, and $\eta$, as shown in Equation \ref{eq:w_t+1},
and define the Counterfactual Loss to be 
\begin{align}
\mathcal{L}_{CF}\left(Z_{t},Z_{t+1},W_{t},\Lambda_{t},\eta\right) & =\mathcal{L}\left(Z_{t+1},W_{t+1}\right)
\end{align}
$\mathcal{L}_{CF}$ is the empirical loss $\mathcal{L}$, where the
weights have already been updated using SGD over the regularized loss
$\mathcal{L}^{\dagger}$. We call this the Counterfactual Loss since
we are asking a counterfactual question: \textit{What would have been
the loss of the network had we updated the weights with respect to
$\mathcal{L}^{\dagger}$}? We will use $\mathcal{L}_{CF}$ to optimize
the regularization coefficients using SGD \textit{while learning the
weights of the network simultaneously} using $\mathcal{L}^{\dagger}$.
We call this technique Regularization Learning, and networks that
employ it \textit{Regularization Learning Networks} (RLNs). 
\begin{thm}
\label{thm:The-gradient-of}The gradient of the Counterfactual loss
with respect to the regularization coefficient is $\frac{\partial\mathcal{L}_{CF}}{\partial\lambda_{t,i}}=-\eta\cdot g_{t+1,i}\cdot r_{t,i}$
\end{thm}
\begin{proof}
$\mathcal{L}_{CF}$ only depends on $\lambda_{t,i}$ through $w_{t+1,i}$,
allowing us to use the chain rule $\frac{\partial\mathcal{L}_{CF}}{\partial\lambda_{t,i}}=\frac{\partial\mathcal{L}_{CF}}{\partial w_{t+1,i}}\cdot\frac{\partial w_{t+1,i}}{\partial\lambda_{t,i}}$.
The first multiplier is the gradient $g_{t+1,i}$. Regarding the second
multiplier, from Equation \ref{eq:w_t+1} we see that only $r_{t,i}$
depends on $\lambda_{t,i}$. Combining with Equation \ref{eq:r} leaves
us with: 

\begin{gather*}
\frac{\partial w_{t+1,i}}{\partial\lambda_{t,i}}=\frac{\partial}{\partial\lambda_{t,i}}\left(w_{t,i}-\eta\cdot\left(g_{t,i}+r_{t,i}\right)\right)=-\eta\cdot\frac{\partial r_{t,i}}{\partial\lambda_{t,i}}=\\
=-\eta\cdot\frac{\partial}{\partial\lambda_{t,i}}\left(\exp\left(\lambda_{t,i}\right)\cdot\frac{\partial\left\Vert w_{t,i}\right\Vert }{\partial w_{t,i}}\right)=-\eta\cdot\exp\left(\lambda_{t,i}\right)\cdot\frac{\partial\left\Vert w_{t,i}\right\Vert }{\partial w_{t,i}}=-\eta\cdot r_{t,i}
\end{gather*}
\end{proof}
\begin{wrapfigure}{I}{0.5\columnwidth}%
\vspace{-20pt}\includegraphics[width=1\linewidth]{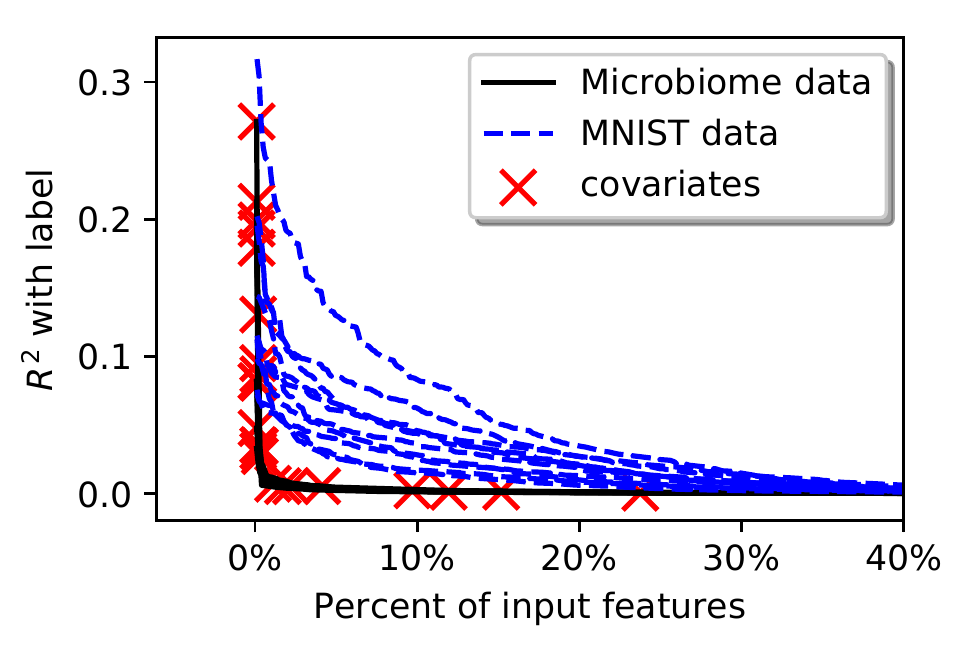}\caption{The input features, sorted by their $R^{2}$ correlation to the label.
We display the microbiome dataset, with the covariates marked, in
comparison the MNIST dataset\cite{lecun1998mnist}.}
\label{wrap:input_feature_importance}\vspace{-20pt}\end{wrapfigure}%

Theorem \ref{thm:The-gradient-of} gives us the update rule $\lambda_{t+1,i}=\lambda_{t,i}-\nu\cdot\frac{\partial\mathcal{L}_{CF}}{\partial\lambda_{t,i}}=\lambda_{t,i}+\nu\cdot\eta\cdot g_{t+1,i}\cdot r_{t,i}$,
where $\nu$ is the learning rate of the regularization coefficients.

Intuitively, the gradient of the Counterfactual Loss has an opposite
sign to the product of $g_{t+1,i}$ and $r_{t,i}$. Comparing this
result with Equation \ref{eq:w_t+1}, this means that when $g_{t+1,i}$
and $r_{t,i}$ agree in sign, the regularization helps reduce the
loss, and we can strengthen it by increasing $\lambda_{t,i}$. When
they disagree, this means that the regularization hurts the performance
of the network, and we should relax it for this weight.

The size of the Counterfactual gradient is proportional to the product
of the sizes of $g_{t+1,i}$ and $r_{t,i}$. When $g_{t+1,i}$ is
small, $w_{t+1,i}$ does not affect the loss $\mathcal{L}$ much,
and when $r_{t,i}$ is small, $\lambda_{t,i}$ does not affect $w_{t+1,i}$
much. In both cases, $\lambda_{t,i}$ has a small effect on $\mathcal{L}_{CF}$.
Only when both $r_{t,i}$ is large (meaning that $\lambda_{t,i}$
affects $w_{t+1}$), and $g_{t+1,i}$ is large (meaning that $w_{t+1}$
affects $\mathcal{L}$), $\lambda_{t,i}$ has a large effect on $\mathcal{L}_{CF}$,
and we get a large gradient $\frac{\partial\mathcal{L}_{CF}}{\partial\lambda_{t,i}}$. 

At the limit of many training iterations, $\lambda_{t,i}$ tends to
continuously decrease. We try to give some insight to this dynamics
in the supplementary material. To address this issue, we project the
regularization coefficients onto a simplex after updating them:
\begin{align}
\widetilde{\lambda}_{t+1,i} & =\lambda_{t,i}+\nu\cdot\eta\cdot g_{t+1,i}\cdot r_{t,i}\\
\lambda_{t+1,i} & =\widetilde{\lambda}_{t+1,i}+\left(\theta-\frac{\sum_{j=1}^{n}\widetilde{\lambda}_{t+1,j}}{n}\right)
\end{align}
where $\theta$ is the normalization factor of the regularization
coefficients, a hyperparameter of the network tuned using cross-validation.
This results in a zero-sum game behavior in the regularization, where
a relaxation in one edge allows us to strengthen the regularization
in other parts of the network. This could lead the network to assign
a modular regularization profile, where uninformative connections
are heavily regularized and informative connection get a very relaxed
regularization, which might boost performance on datasets with non-distributed
representation such as tabular datasets. The full algorithm is described
in the supplementary material.

\begin{figure}[H]
\centering{}\includegraphics[width=1\linewidth]{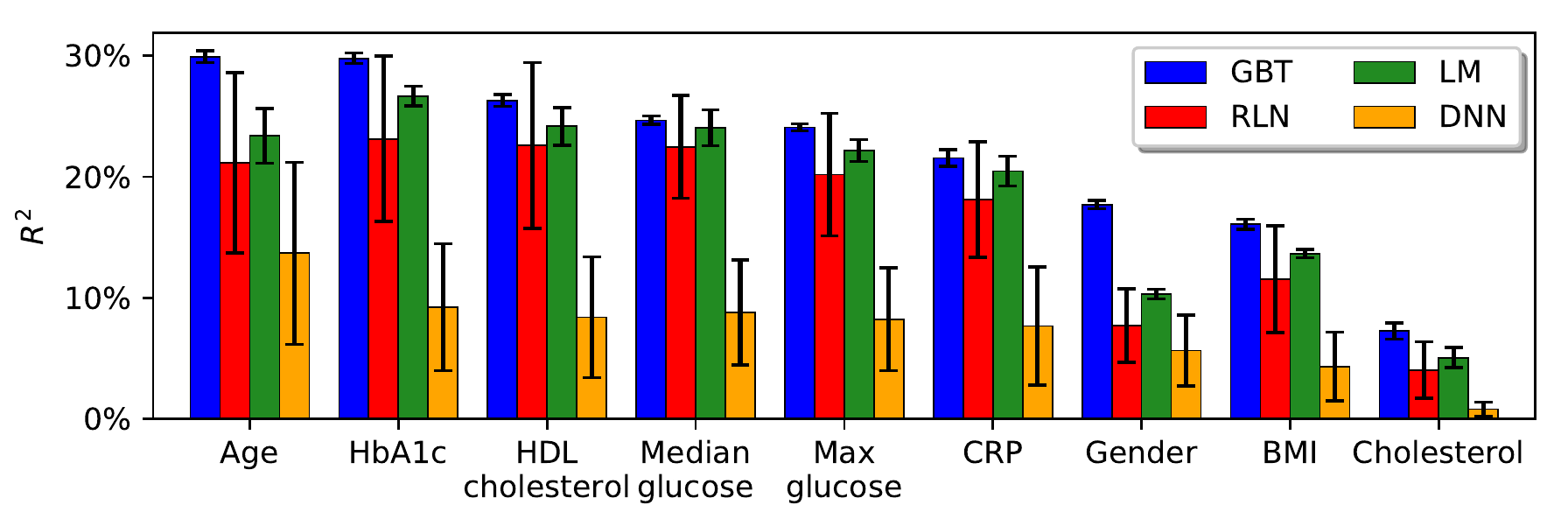}\caption{Prediction of traits using microbiome data and covariates, given as
the overall explained variance ($R^{2}$).}
\label{wrap:first_results}\vspace{-20pt}
\end{figure}

\section{Experiments\label{sec:Experiments}}

We demonstrate the performance of our method on the problem of predicting
human traits from gut microbiome data and basic covariates (age, gender,
BMI). The human gut microbiome is the collection of microorganisms
found in the human gut and is composed of trillions of cells including
bacteria, eukaryotes, and viruses. In recent years, there have been
major advances in our understanding of the microbiome and its connection
to human health. Microbiome composition is determined by DNA sequencing
human stool samples that results in short (75-100 basepairs) DNA reads.
By mapping these short reads to databases of known bacterial species,
we can deduce both the source species and gene from which each short
read originated. Thus, upon mapping a collection of different samples,
we obtain a matrix of estimated relative species abundances for each
person and a matrix of the estimated relative gene abundances for
each person. Since these features have varying relative importance
(Figure \ref{wrap:input_feature_importance}), we expected GBTs to
outperform DNNs on these tasks.

We sampled 2,574 healthy participants for which we measured, in addition
to the gut microbiome, a collection of different traits, including
important disease risk factors such as cholesterol levels and BMI.
Finding associations between these disease risk factors and the microbiome
composition is of

\begin{wrapfigure}{I}{0.45\columnwidth}%
\begin{centering}
\vspace{-20pt}\includegraphics[width=1\linewidth,height=6.5cm]{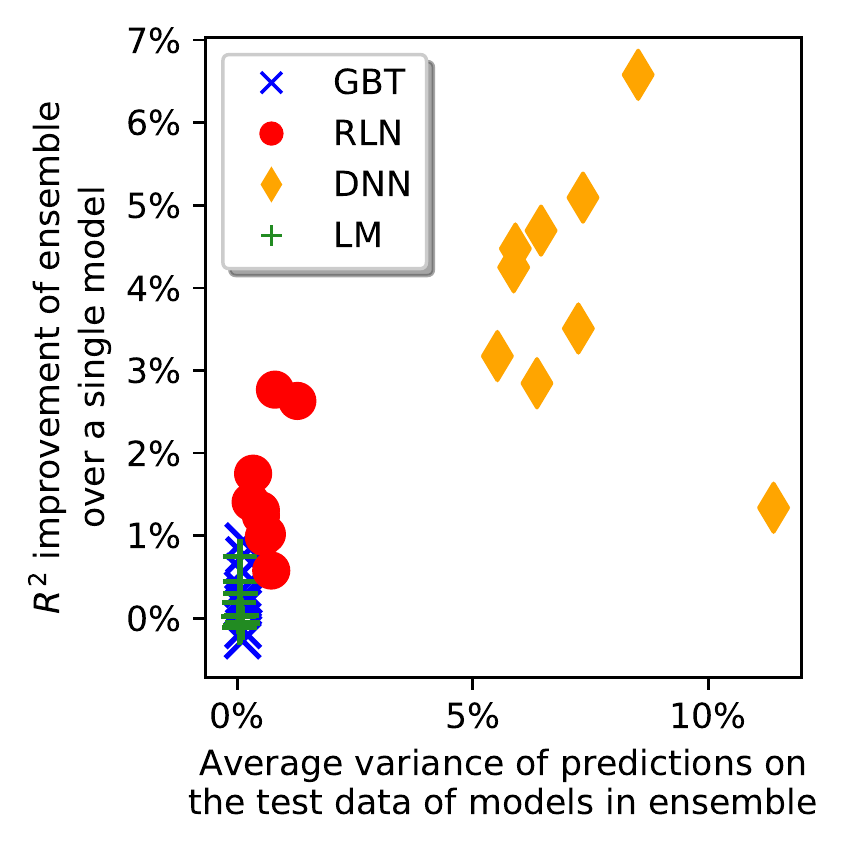}
\par\end{centering}
\caption{For each model type and trait, we took the 10 best performing models,
based on their validation performance, and calculated the average
variance of the predicted test samples, and plotted it against the
improvement in $R^{2}$ obtained when training ensembles of these
models. Note that models that have a high variance in their prediction
benefit more from the use of ensembles. As expected, DNNs gain the
most from ensembling.}
\label{wrap:ensemble_improvement}\vspace{-50pt}\end{wrapfigure}%

great scientific interest, and can raise novel hypotheses about the
role of the microbiome in disease. We tested 4 types of models: RLN,
GBT, DNN, and \textit{Linear Models} (LM). The full list of hyperparameters,
the setting of the training of the models and the ensembles, as well
as the description of all the input features and the measured traits,
can be found in the supplementary material.

\begin{figure}[b]
\begin{centering}
\vspace{-20pt}\includegraphics[width=1\columnwidth]{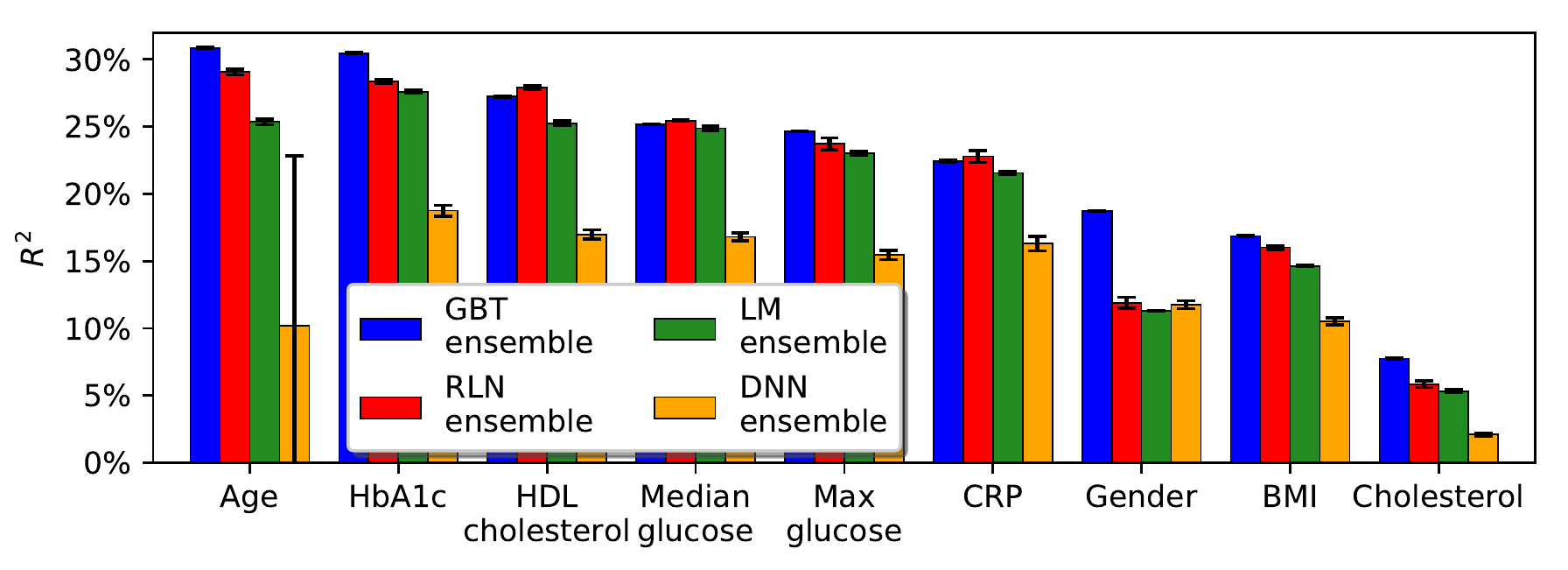}
\par\end{centering}
\caption{Ensembles of different predictors. }
\label{ensembles from a single model}
\end{figure}

\section{Results\label{sec:Results}}

\begin{figure}[t]
\centering{}\includegraphics[width=1\linewidth]{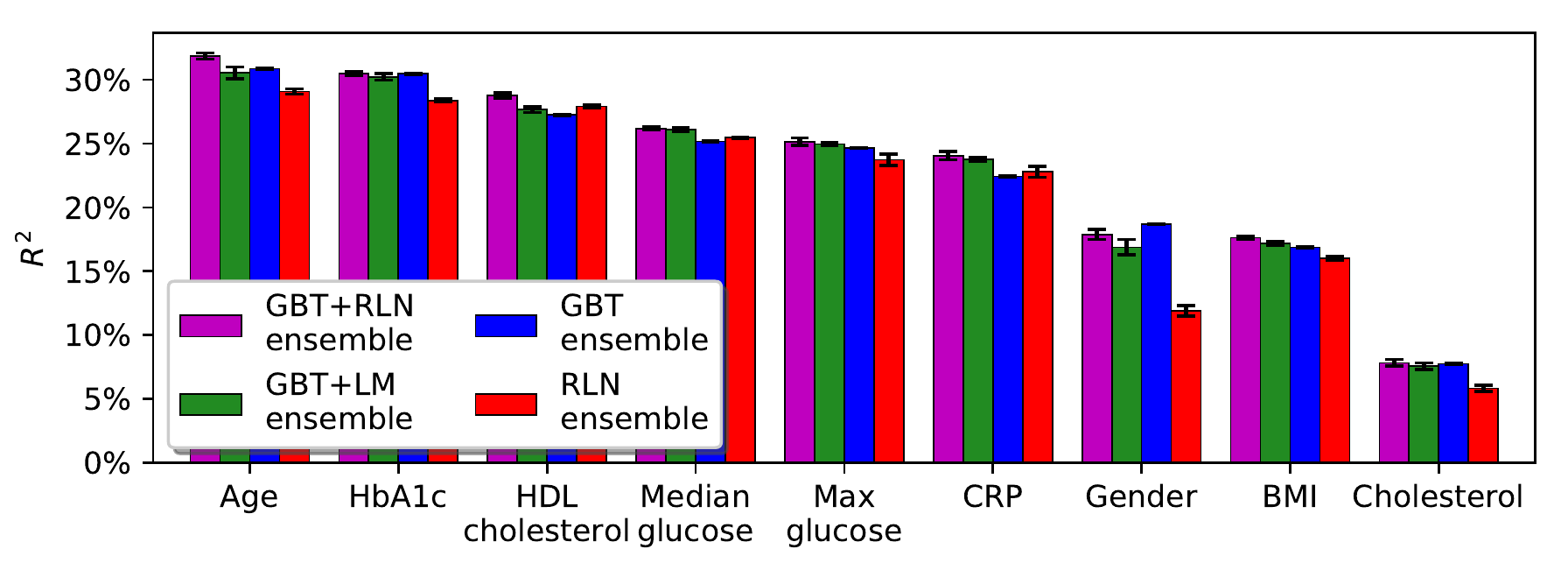}\caption{Results of various ensembles that are each composed of different types
of models.}
\label{final result figure}
\end{figure}
When running each model separately, GBTs achieve the best results
on all of the tested traits, but it is only significant on $3$ of
them (Figure \ref{wrap:first_results}). DNNs achieve the worst results,
with $15\%\pm1\%$ less explained variance than GBTs on average. RLNs
significantly and substantially improve this by a factor of \textbf{$\boldsymbol{2.57\pm0.05}$},
and achieve only $2\%\pm2\%$ less explained variance than GBTs on
average.

\begin{table}[t]
\begin{tabular}{>{\raggedright}m{1.3cm}>{\centering}p{2.3cm}>{\centering}p{2.2cm}>{\centering}p{2.4cm}>{\centering}p{2cm}>{\centering}p{0.9cm}}
\multicolumn{1}{>{\raggedright}p{1.3cm}}{Trait} & RLN + GBT  & LM + GBT  & GBT & RLN & Max\tabularnewline
\midrule
Age & $\boldsymbol{\underline{31.9\%\pm0.2\%}}$ & $30.5\%\pm0.5\%$ & $30.9\%\pm0.1\%$ & $29.1\%\pm0.2\%$ &  $31.9\%$\tabularnewline
\midrule
HbA1c & $\boldsymbol{30.5\%\pm0.2\%}$ & $30.2\%\pm0.3\%$ & $30.5\%\pm0.04\%$ & $28.4\%\pm0.1\%$ & $30.5\%$\tabularnewline
\midrule
\multicolumn{1}{>{\raggedright}p{1.6cm}}{HDL cholesterol} & $\boldsymbol{\underline{28.8\%\pm0.2\%}}$ & $27.7\%\pm0.2\%$ & $27.2\%\pm0.04\%$ & $27.9\%\pm0.1\%$ & $28.8\%$\tabularnewline
\midrule
\multicolumn{1}{>{\raggedright}p{1.3cm}}{Median glucose} & $\boldsymbol{26.2\%\pm0.1\%}$ & $26.1\%\pm0.1\%$ & $25.2\%\pm0.04\%$ & $25.5\%\pm0.1\%$ & $26.2\%$\tabularnewline
\midrule
\multicolumn{1}{>{\raggedright}p{1.3cm}}{Max glucose} & $\boldsymbol{25.2\%\pm0.3\%}$ & $25.0\%\pm0.1\%$ & $24.6\%\pm0.03\%$ & $23.7\%\pm0.4\%$ & $25.2\%$\tabularnewline
\midrule
\multirow{1}{1.3cm}{CRP} & $\boldsymbol{24.0\%\pm0.3\%}$ & $23.7\%\pm0.2\%$ & $22.4\%\pm0.1\%$ & $22.8\%\pm0.4\%$ & $24.0\%$\tabularnewline
\midrule
\multirow{1}{1.3cm}{Gender} & $17.9\%\pm0.4\%$ & $16.9\%\pm0.6\%$ & $\boldsymbol{18.7\%\pm0.03\%}$ & $11.9\%\pm0.4\%$ & $18.7\%$\tabularnewline
\midrule
\multirow{1}{1.3cm}{BMI} & $\boldsymbol{\underline{17.6\%\pm0.1\%}}$ & $17.2\%\pm0.2\%$ & $16.9\%\pm0.04\%$ & $16.0\%\pm0.1\%$ & $17.6\%$\tabularnewline
\midrule
\multirow{1}{1.3cm}{Cholesterol} & $\boldsymbol{7.8\%\pm0.3\%}$ & \textbf{$7.6\%\pm0.3\%$} & $7.8\%\pm0.1\%$ & $5.8\%\pm0.2\%$ & $7.8\%$\tabularnewline
\bottomrule
\end{tabular}

\caption{Explained variance ($R^{2}$) of various ensembles with different
types of models. Only the 4 ensembles that achieved the best results
are shown. The best result for each trait is highlighted, and underlined
if it outperforms significantly all other ensembles.}
\label{ensemble_results}\vspace{-20pt}
\end{table}
Constructing an ensemble of models is a powerful technique for improving
performance, especially for models which have high variance, like
neural networks in our task. As seen in Figure \ref{wrap:ensemble_improvement},
the average variance of predictions of the top 10 models of RLN and
DNN is $1.3\%\pm0.6\%$ and $14\%\pm3\%$ respectively, while the
variance of predictions of the top 10 models of LM and GBT is only
$0.13\%\pm0.05\%$ and $0.26\%\pm0.02\%$, respectively. As expected,
the high variance of RLN and DNN models allows ensembles of these
models to improve the performance over a single model by $1.5\%\pm0.7\%$
and $4\%\pm1\%$ respectively, while LM and GBT only improve by $0.2\%\pm0.3\%$
and $0.3\%\pm0.4\%$, respectively. Despite the improvement, DNN ensembles
still achieve the worst results on all of the traits except for \textit{Gender}
and achieve results $9\%\pm1\%$ lower than GBT ensembles (Figure
\ref{ensembles from a single model}). In comparison, this improvement
allows RLN ensembles to outperform GBT ensembles on \textit{HDL cholesterol,
Median glucose, }and\textit{ CRP}, and to obtain results $8\%\pm1\%$
higher than DNN ensembles and only $1.4\%\pm0.1\%$ lower than GBT
ensembles.

Using ensemble of different types of models could be even more effective
because their errors are likely to be even more uncorrelated than
ensembles from one type of model. Indeed, as shown in Figure \ref{final result figure},
the best performance is obtained with an ensemble of RLN and GBT,
which achieves the best results on all traits except \textit{Gender},
and outperforms all other ensembles significantly on \textit{Age},
\textit{BMI}, and \textit{HDL cholesterol} (Table \ref{ensemble_results})

\section{Analysis\label{sec:Analysis}}

\begin{wrapfigure}{I}{0.8\columnwidth}%
\begin{centering}
\vspace{-25pt}\subfloat[]{\includegraphics[width=0.46\linewidth]{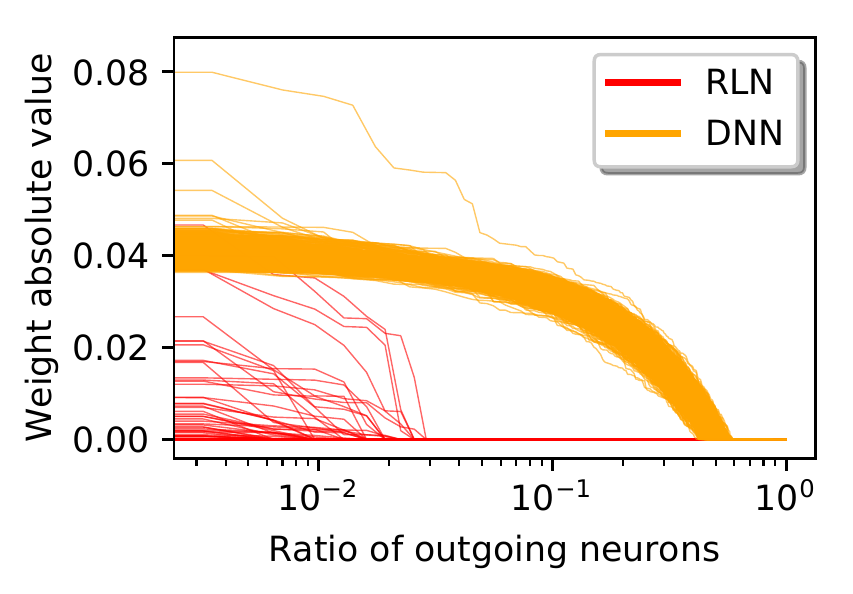}\label{weight distribution}}\subfloat[]{\includegraphics[width=0.46\linewidth]{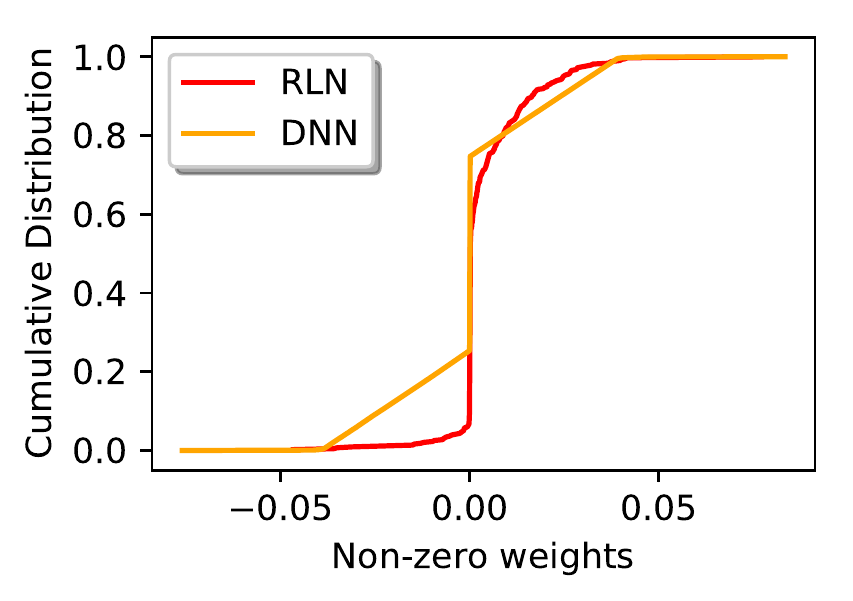}\label{non-zero weights cummulative distribution}}
\par\end{centering}
\caption{a) Each line represents an input feature in a model. The values of
each line are the absolute values of its outgoing weights, sorted
from greatest to smallest. Noticeably, only $12\%$ of the input features
have any non-zero outgoing edge in the RLN model. b) The cumulative
distribution of non-zero outgoing weights for the input features for
different models. Remarkably, the distribution of non-zero weights
is quite similar for the two models.}

\vspace{-7pt}\end{wrapfigure}%

We next sought to examine the effect that our new type of regularization
has on the learned networks. Strikingly, we found that RLNs are extremely
sparse, even compared to $L_{1}$ regulated networks. To demonstrate
this, we took the hyperparameter setting that achieved the best results
on the \textit{HbA1c} task for the DNN and RLN models and trained
a single network on the entire dataset. Both models achieved their
best hyperparameter setting when using $L_{1}$ regularization. Remarkably,
$82\%$ of the input features in the RLN do not have any non-zero
outgoing edges, while all of the input features have at least one
non-zero outgoing edge in the DNN (Figure \ref{weight distribution}).
A possible explanation could be that the RLN was simply trained using
a stronger regularization coefficients, and increasing the value of
$\lambda$ for the DNN model would result in a similar behavior for
the DNN, but in fact the RLN was obtained with an average regularization
coefficient of $\theta=-6.6$ while the DNN model was trained using
a regularization coefficient of $\lambda=-4.4$. Despite this extreme
sparsity, the non zero weights are not particularly small and have
a similar distribution as the weights of the DNN (Figure \ref{non-zero weights cummulative distribution}).

We suspect that the combination of a sparse network with large weights
allows RLNs to achieve their improved performance, as our dataset
includes features with varying relative importance. To show this,
we re-optimized the hyperparameters of the DNN and RLN models after
removing the covariates from the datasets. The covariates are very
important features (Figure \ref{wrap:input_feature_importance}),
and removing them would reduce the variability in relative importance.
As can be seen in Figure \ref{only microbiome-results}, even without
the covariates, the RLN and GBT ensembles still achieve the best results
on $5$ out of the $9$ traits. However, this improvement is less
significant than when adding the covariates, where RLN and GBT ensembles
achieve the best results on $8$ out of the $9$ traits. RLNs still
significantly outperform DNNs, achieving explained variance higher
by $2\%\pm1\%$, but this is significantly smaller than the $9\%\pm2\%$
improvement obtained when adding the covariates (Figure \ref{only microbiome-rln-improvement}).
We speculate that this is because RLNs particularly shine when features
have very different relative importances.

\begin{figure}[t]
\begin{centering}
\subfloat[]{\includegraphics[width=1\linewidth]{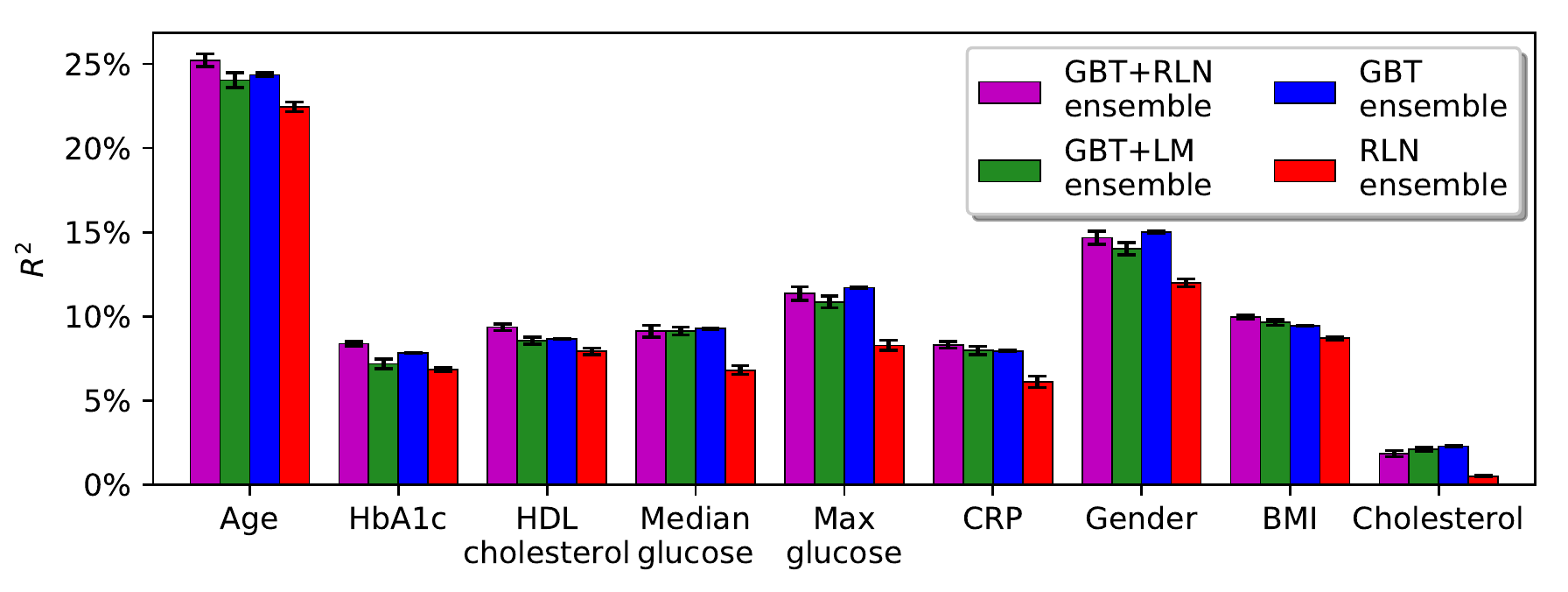}\label{only microbiome-results}}
\par\end{centering}
\centering{}\subfloat[]{\includegraphics[width=1\linewidth]{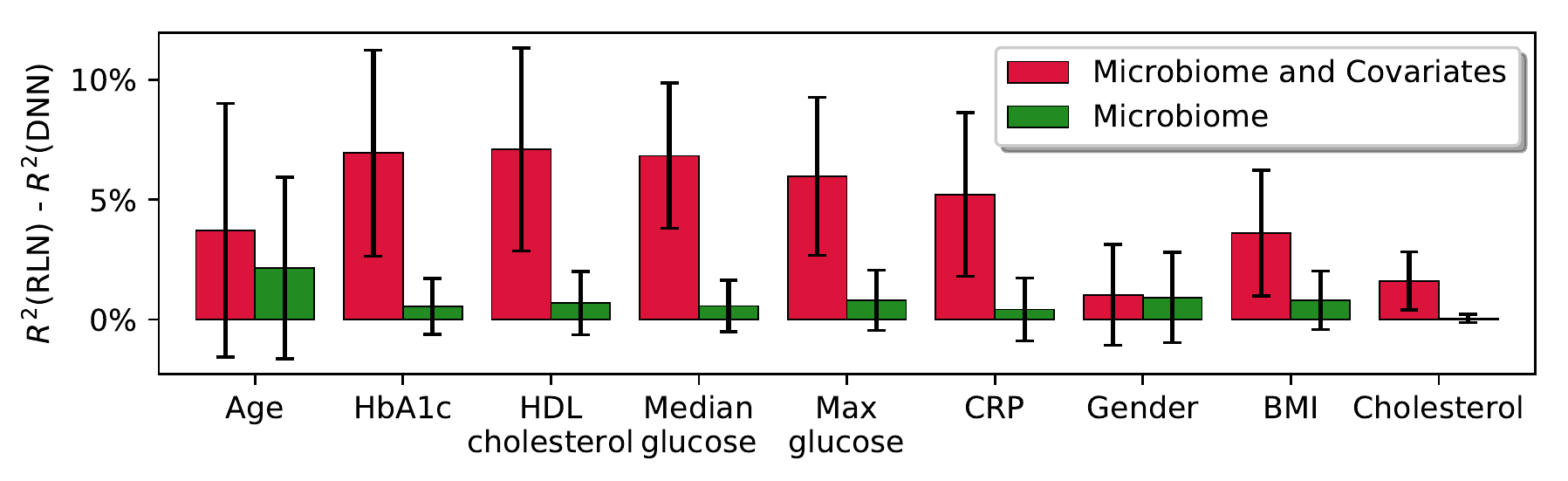}\label{only microbiome-rln-improvement}}\caption{a) Training our models without adding the covariates. b) The relative
improvement RLN achieves compared to DNN for different input features.}
\label{only microbiome}\vspace{-15pt}
\end{figure}
\begin{figure}[t]
\includegraphics[width=1\linewidth]{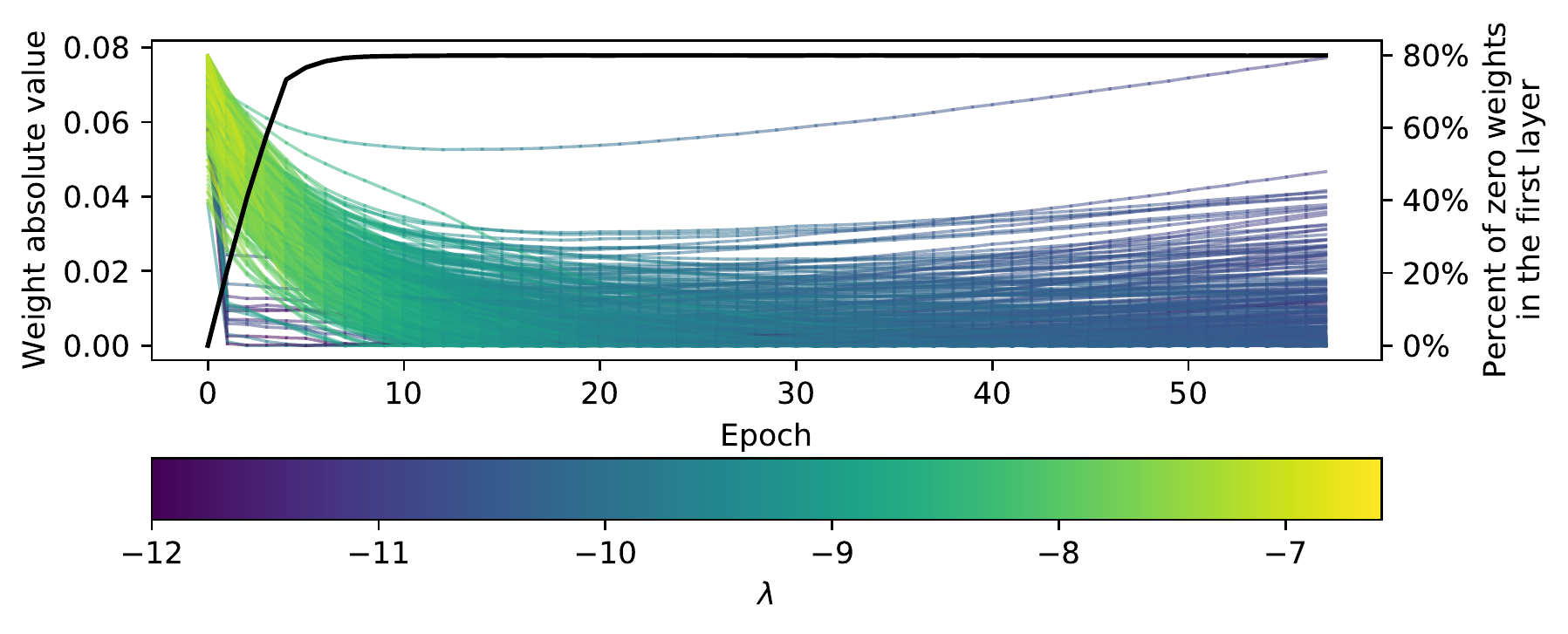}\caption{On the left axis, shown is the traversal of edges of the first layer
that finished the training with a non-zero weight in the $w$, $\lambda$
space. Each colored line represents an edge, its color represents
its regularization, with yellow lines having strong regularization.
On the right axis, the black line plots the percent of zero weight
edges in the first layer during training.}

\centering{}\label{edge traversal}\vspace{-20pt}
\end{figure}
To understand what causes this interesting structure, we next explored
how the weights in RLNs change during training. During training, each
edge performs a traversal in the $w$, $\lambda$ space. We expect
that when $\lambda$ decreases and the regularization is relaxed,
the absolute value of $w$ should increase, and vice versa. In Figure
\ref{edge traversal}, we can see that $\boldsymbol{\underline{99.9\%}}$
of the edges of the first layer finish the training with a zero value.
There are still $434$ non-zero edges in the first layer due to the
large size of the network. This is not unique to the first layer,
and in fact, $\boldsymbol{\underline{99.8\%}}$ of the weights of
the entire network have a zero value by the end of the training. The
edges of the first layer that end up with a non-zero weight are decreasing
rapidly at the beginning of the training because of the regularization,
but during the first 10-20 epochs, the network quickly learns better
regularization coefficients for its edges. The regularization coefficients
are normalized after every update, hence by applying stronger regularization
on some edges, the network is allowed to have a more relaxed regularization
on other edges and consequently a larger weight. By epoch 20, the
edges of the first layer that end up with a non-zero weight have an
average regularization coefficient of $-9.4$, which is significantly
smaller than their initial value $\theta=-6.6$. These low values
pose effectively no regularization, and their weights are updated
primarily to minimize the empirical loss component of the loss function,
$\mathcal{L}$.

Finally, we reasoned that since RLNs assign non-zero weights to a
relatively small number of inputs, they may be used to provide insights
into the inputs that the model found to be more important for generating
its predictions using Garson's algorithm \cite{Garson:1991:INC:129449.129452}.
There has been important progress in recent years in sample-aware
model interpretability techniques in DNNs \cite{Shrikumar,Sundararajan},
but tools to produce sample-agnostic model interpretations are lacking
\cite{Hooker}.\footnote{The sparsity of RLNs could be beneficial for sample-aware model interpretability
techniques such as \cite{Shrikumar,Sundararajan}. This was not examined
in this paper.} Model interpretability is particularly important in our problem for
obtaining insights into which bacterial species contribute to predicting
each trait.

Evaluating feature importance is difficult, especially in domains
in which little is known such as the gut microbiome. One possibility
is to examine the information it supplies. In Figure \ref{Feature importance entropy}
we show the feature importance achieved through this technique using
RLNs and DNNs. While the importance in DNNs is almost constant and
does not give any meaningful information about the specific importance
of the features, the importance in RLNs is much more meaningful, with
entropy of the $4.6$ bits for the RLN importance, compared to more
than twice for the DNN importance, $9.5$ bits.

Another possibility is to evaluate its consistency across different
instantiations of the model. We expect that a good feature importance
technique will give similar importance distributions regardless of
instantiation. We trained 10 instantiations for each model and phenotype
and evaluated their feature importance distributions, for which we
calculated the Jensen-Shannon divergence. In Figure \ref{feature importance divergence}
we see that RLNs have divergence values $48\%\pm1\%$ and $54\%\pm2\%$
lower than DNNs and LMs respectively. This is an indication that Garson\textquoteright s
algorithm results in meaningful feature importances in RLNs. We list
of the 5 most important bacterial species for different traits in
the supplementary material.

\begin{figure}[t]
\subfloat[]{\includegraphics[width=0.3\linewidth]{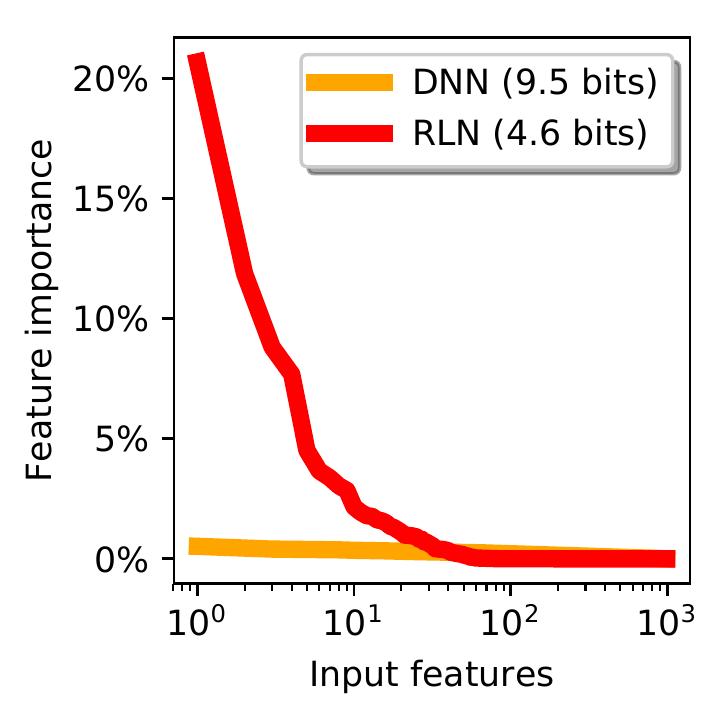}\label{Feature importance entropy}}\subfloat[]{\includegraphics[width=0.7\linewidth]{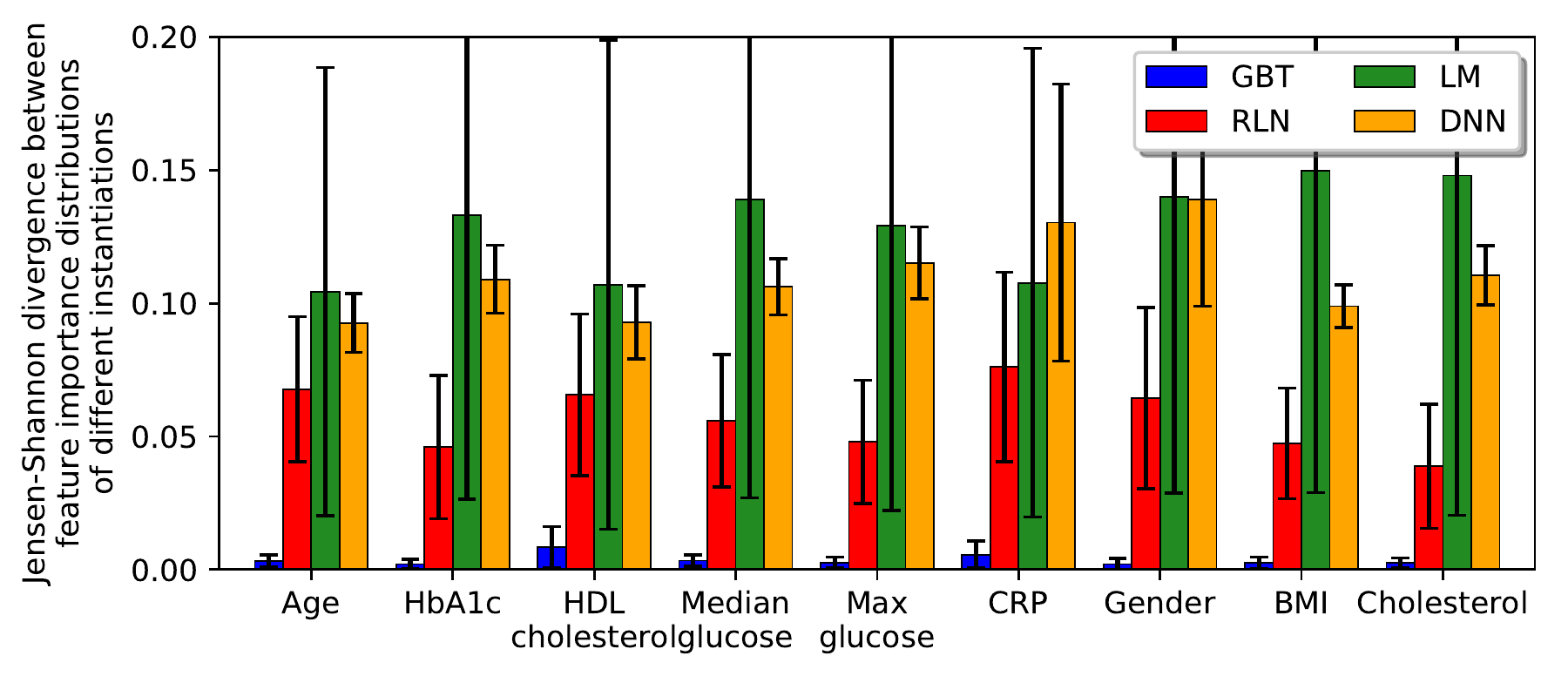}\label{feature importance divergence}}

\caption{a) The input features, sorted by their importance, in a DNN and RLN
models. b) The Jensen-Shannon divergence between the feature importance
of different instantiations of a model.}

\vspace{-15pt}
\end{figure}

\section{Conclusion}

In this paper, we explore the learning of datasets with non-distributed
representation, such as tabular datasets. We hypothesize that modular
regularization could boost the performance of DNNs on such tabular
datasets. We introduce the \textit{Counterfactual Loss}, \emph{$\mathcal{L}_{CF}$},
and \textit{Regularization Learning Networks} (RLNs) which use the
Counterfactual Loss to tune its regularization hyperparameters efficiently
during learning together with the learning of the weights of the network.

We test our method on the task of predicting human traits from covariates
and microbiome data and show that RLNs significantly and substantially
improve the performance over classical DNNs, achieving an increased
explained variance by a factor of $2.75\pm0.05$ and comparable results
with GBTs. The use of ensembles further improves the performance of
RLNs, and ensembles of RLN and GBT achieve the best results on all
but one of the traits, and outperform significantly any other ensemble
not incorporating RLNs on $3$ of the traits.

We further explore RLN structure and dynamics and show that RLNs learn
extremely sparse networks, eliminating $99.8\%$ of the network edges
and $82\%$ of the input features. In our setting, this was achieved
in the first 10-20 epochs of training, in which the network learns
its regularization. Because of the modularity of the regularization,
the remaining edges are virtually not regulated at all, achieving
a similar distribution to a DNN. The modular structure of the network
is especially beneficial for datasets with high variability in the
relative importance of the input features, where RLNs particularly
shine compared to DNNs. The sparse structure of RLNs lends itself
naturally to model interpretability, which gives meaningful insights
into the relation between features and the labels, and may itself
serve as a feature selection technique that can have many uses on
its own \cite{Goodman}.

Besides improving performance on tabular datasets, another important
application of RLNs could be learning tasks where there are multiple
data sources, one that includes features with high variability in
the relative importance, and one which does not. To illustrate this
point, consider the problem of detecting pathologies from medical
imaging. DNNs achieve impressive results on this task \cite{Suzuki},
but in real life, the imaging is usually accompanied by a great deal
of tabular metadata in the form of the electronic health records of
the patient. We would like to use both datasets for prediction, but
different models achieve the best results on each part of the data.
Currently, there is no simple way to jointly train and combine the
models. Having a DNN architecture such as RLN that performs well on
tabular data will thus allow us to jointly train a network on both
of the datasets natively, and may improve the overall performance.

\subsubsection*{Acknowledgments}

We would like to thank Ron Sender, Eran Kotler, Smadar Shilo, Nitzan
Artzi, Daniel Greenfeld, Gal Yona, Tomer Levy, Dror Kaufmann, Aviv
Netanyahu, Hagai Rossman, Yochai Edlitz, Amir Globerson and Uri Shalit
for useful discussions.

{\small{}\bibliographystyle{plain}
\bibliography{Regularization_Learning,Additional}
}{\small\par}
\end{document}